\documentclass{article}

\usepackage[authoryear,compress,square]{natbib}
 \usepackage{textcomp, libertine} 
\usepackage{microtype}
\usepackage{graphicx}
\usepackage{subfigure}
\usepackage{tablefootnote}

\usepackage{xr}
\makeatletter
\newcommand*{\addFileDependency}[1]{
  \typeout{(#1)}
  \@addtofilelist{#1}
  \IfFileExists{#1}{}{\typeout{No file #1.}}
}
\makeatother

\usepackage{booktabs} 
\usepackage{amsmath,amsthm,amssymb,burt2}
\usepackage{hyperref}  

\usepackage{xargs}
\usepackage[colorinlistoftodos,prependcaption,textsize=tiny]{todonotes}
\newcommandx{\mvdwnote}[2][1=]{\todo[linecolor=blue,backgroundcolor=blue!25,bordercolor=blue,#1]{MvdW: #2}}

\usepackage[accepted]{icml2019}

\usepackage{MarkMathCmds}

\usepackage{cleveref}
\creflabelformat{equation}{#2\textup{#1}#3}
\renewcommand{\KL}[2]{\mathrm{KL}\left( {#1} \| {#2} \right)}

\icmltitlerunning{Rates of Convergence for Sparse Variational Gaussian Process Regression}
\setcitestyle{square}
\begin{document}

\twocolumn[
\icmltitle{Rates of Convergence for Sparse Variational Gaussian Process Regression}




\begin{icmlauthorlist}
\icmlauthor{David R. Burt}{cam}
\icmlauthor{Carl Edward Rasmussen}{cam,prowler}
\icmlauthor{Mark van der Wilk}{prowler}
\end{icmlauthorlist}

\icmlaffiliation{cam}{University of Cambridge, Cambridge, UK}
\icmlaffiliation{prowler}{PROWLER.io, Cambridge, UK}

\icmlcorrespondingauthor{David R. Burt}{drb62@cam.ac.uk}

\icmlkeywords{Gaussian Processes, Variational Inference}

\vskip 0.3in
]



\printAffiliationsAndNotice{}

\begin{abstract}
Excellent variational approximations to Gaussian process posteriors have been developed which avoid the $\BigO\left(N^3\right)$ scaling with dataset size $N$. They reduce the computational cost to $\BigO\left(NM^2\right)$, with $M\ll N$ the number of \emph{inducing variables}, which summarise the process. While the computational cost seems to be linear in $N$, the true complexity of the algorithm depends on how $M$ must increase to ensure a certain quality of approximation. We show that with high probability the KL divergence can be made arbitrarily small by growing $M$ more slowly than $N$. A particular case is that for regression with normally distributed inputs in D-dimensions with the Squared Exponential kernel, $M=\BigO(\log^D N)$ suffices. Our results show that as datasets grow, Gaussian process posteriors can be approximated cheaply, and provide a concrete rule for how to increase $M$ in continual learning scenarios.
\end{abstract}

\section{Introduction}\label{sec:intro}

Gaussian processes (GPs) \citep{rasmussen_gaussian_2005} are distributions over functions that are convenient priors in Bayesian models. They can be seen as infinitely wide neural networks \citep{neal1996bayesian}, and are particularly popular in regression models, as they produce good uncertainty estimates, and have closed-form expressions for the posterior and marginal likelihood. The most well known drawback of GP regression is the computational cost of the exact calculation of these quantities, which scales as $\BigO\left(N^3\right)$ in time and $\BigO\left(N^2\right)$ in memory where $N$ is the number of training examples. Low-rank approximations \citep{quin2005unifying} choose $M$ \emph{inducing variables} which summarise the entire posterior, reducing the cost to $\BigO\left(NM^2 + M^3\right)$ time and $\BigO\left(NM + M^2\right)$ memory.

While the computational cost of adding inducing variables is well understood, results on how many are needed to achieve a good approximation are lacking. As the dataset size increases, we cannot expect to keep the capacity of the approximation constant without the quality deteriorating. 
Taking into account the rate at which $M$ must increase with $N$ to achieve a particular approximation accuracy, as well as the cost of initializing or optimizing the inducing points, determines a more realistic sense of the costs of scaling Gaussian processes.

Approximate GPs are often trained using variational inference \citep{titsias_variational_2009}, which minimizes the KL divergence from an approximate posterior to the full posterior process \citep{matthews_sparse_2016}. We use this KL divergence as our metric for the approximate posterior's quality. We show that under intuitive assumptions the number of inducing variables only needs to grow at a sublinear rate for the KL between the approximation and the posterior to go to zero. This shows that very sparse approximations can be used for large datasets, without introducing much bias into hyperparameter selection through evidence lower bound (ELBO) maximisation, and with approximate posteriors that are accurate in terms of their prediction and uncertainty.

The core idea of our proof is to use upper bounds on the KL divergence that depend on the quality of a Nystr\"om approximation to the data covariance matrix. Using existing results, we show this error can be understood in terms of the spectrum of an infinite-dimensional integral operator. In the case of stationary kernels, our main result proves that priors with smoother sample functions, and datasets with more concentrated inputs admit sparser approximations.

\paragraph{Main results}
We assume that training inputs are drawn i.i.d.~from a fixed distribution, and prove bounds of the form
\[
\KL{Q}{\hat{P}} \leq \BigO\left(\frac{g(M,N)}{\sn^2\delta}\left(1+\frac{c\|\bfy\|_2^2}{\sn^2}\right)+ N\epsilon \right)
\]
with probability at least $1-\delta$, where $\hat{P}$ is the posterior Gaussian process, $Q$ is a variational approximation, and $\bfy$ are the training targets. The function $g(M,N)$ depends on both the kernel and input distribution, and grows linearly in $N$ and generally decays rapidly in $M$. The quality of the initialization determines $\epsilon$, which can be made arbitrarily small (e.g.~an inverse power of $N$) at some additional computational cost. 
\Cref{thm:boundsforeigfeats,thm:avgeigencase} give results of this form for a collection of inducing variables defined using spectral information, \cref{thm:generalexplicit,thm:avgcase} hold for inducing points. 

\section{Background and notation}\label{sec:background}
\subsection{Gaussian process regression}
We consider Gaussian process regression, where we observe \emph{training data}, $\data=\{\bfx_i ,y_i\}_{i=1}^N$ with $\bfx_i \in \mcX$ and $y_i \in \R.$
Our goal is to predict outputs $y^*$ for new inputs $\vx^*$ while taking into account the uncertainty we have about $f(\cdot)$ due to the limited size of the training set. We follow a Bayesian approach by placing a prior over $f$, and a likelihood to relate $f$ to the observed data through some noise. Our model is
\begin{align*}
    f \sim \mathcal{GP}(\nu(\cdot), k(\cdot, \cdot)), &&
    y_i \!= \!f(\bfx_i) + \epsilon_i, &&
    \epsilon_i \sim \mathcal{N}(0, \sn^2),
\end{align*}
where $\nu: \mcX \to \R$ is the \emph{mean function} and $k: \mcX \times \mcX \to \R$ is the \emph{covariance function}. We take $\nu \equiv 0;$ the general case can be derived similarly after first centering the process. We use the posterior for making predictions, and the marginal likelihood for selecting hyperparameters, both of which have closed-form expressions \citep{rasmussen_gaussian_2005}.
The log marginal likelihood is of particular interest to us, as the quality of its approximation and our posterior approximation is linked. Its form is
\begin{equation}\label{eq:ml}
    \mcL= -\frac{1}{2}\bfy\transpose{\bf K}_n\inv\bfy -\frac{1}{2}\log\detbar{{\bf K}_n} - \frac{N}{2}\log(2\pi) \,,
\end{equation}
where ${\bf K}_n = \Kff + \sn^2 \bfI$, and $\left[\Kff\right]_{i,j}=k(\bfx_i,\bfx_j).$

\subsection{Sparse variational Gaussian process regression}
\label{sec:lower}
While all quantities of interest have analytic expressions, their computation is infeasible for large datasets due to the $\BigO\left(N^3\right)$ time complexity of the determinant and inverse. Many approaches have been proposed that rely on a low-rank approximation to $\Kff$ \citep{quin2005unifying, rahimi_random_2008}, which allow the determinant and inverse to be computed in $\BigO\left(NM^2\right)$, where $M$ is the rank of the approximating matrix.

We consider the variational framework developed by \citet{titsias_variational_2009}, which minimizes the KL divergence from the posterior process to an approximate GP
\begin{equation}
    \GP\left(\vk_{\cdot\vu}\Kuu\inv\vmu, k_{\cdot\cdot} + \vk_{\cdot\vu}\Kuu\inv\left(\boldsymbol{\Sigma} - \Kuu\right)\Kuu\inv\vk_{\vu\cdot}\right) \label{eq:qf} \,,
\end{equation}
where $\left[\vk_{\vu\cdot}\right]_i\!=\!k(\cdot, \vz_i)$, $\left[\Kuf\right]_{m,i}\!:=\!k(\bfz_m,\bfx_i)$ and $\left[\Kuu\right]_{m,n} := k(\bfz_m,\bfz_n)$. This variational distribution is determined through defining the density of the function values $\vu \in \Reals^M$ at \emph{inducing inputs} $Z = \left\{\vz_m\right\}_{m=1}^M$ to be $q(\vu) = \NormDist{\vmu, \boldsymbol{\Sigma}}$. $Z$, $\vmu$, and $\boldsymbol{\Sigma}$ are variational parameters.
\citet{titsias_variational_2009} solved the convex optimization problem for $\vmu$ and $\boldsymbol{\Sigma}$ explicitly, resulting in the evidence lower bound (ELBO):  
\begin{multline}\label{eqn:elbo}
    \mcLl\!=\!-\frac{1}{2}\bfy\transpose{\bf Q}_n^{-1}\bfy - \frac{1}{2}\!\log\detbar{{\bf Q}_n} - \frac{N}{2}\!\log(2\pi)-\frac{t}{2\sn^2} \!\!\!\!\!\!
\end{multline}
where ${\bf Q}_n = \Qff+\sn^2\bfI$, $\Qff=\Kuf\transpose\Kuu^{-1}\Kuf$ and $t=\Tr\left(\Kff-\Qff\right).$ \citet{hensman_gaussian_2013} proposed optimising over $\{\vmu, \boldsymbol{\Sigma}\}$ numerically, which allows minibatch optimization. In both cases, $\mcL = \mcLl + \mathrm{KL}(Q||\hat{P})$ \citep{matthews_sparse_2016}, so any bounds on the KL-divergence we prove hold for \citet{hensman_gaussian_2013} as well, as long as $\{\vmu, \boldsymbol{\Sigma}\}$ is at the optimum value.

\Citet{titsias_variational_2009} suggests jointly maximizing the ELBO (\cref{eqn:elbo}) w.r.t.~the variational and hyperparameters.
This comes at the cost of introducing bias in hyperparameter estimation \citep{turner_sahani_2011}, notably the overestimation of the $\sn^2$ \citep{bauer_understanding_2016}. Adding inducing points reduces the KL gap \citep{titsias_variational_2009}, and the bias is practically eliminated when enough inducing variables are used. 

\subsection{Interdomain inducing features}
\citet{lazaro-gredilla_inter-domain_2009} showed that one can specify the distribution $q(\vu),$ on integral transformations of $f(\cdot).$ Using these \emph{interdomain} inducing variables can lead to sparser representations, or computational benefits \citep{hensman2018variational}. Interdomain inducing variables are defined by
\[
u_m = \int_{\mathcal{X}} f(\bfx) g(\bfx; \vz_m)d\bfx \,.
\]
When $g(\bfx; \vz_m)= \delta(\bfx-\bfz_m)$ the $u_m$ are inducing points. Interdomain features require replacing $\vk_{\vu\cdot}$ and $\Kuu$ in \cref{eq:qf} with integral transforms of the kernel. In later sections, we investigate particular interdomain transformations with interesting convergence properties.

\subsection{Upper bounds on the marginal likelihood}\label{sec:upper}
Combined with \cref{eqn:elbo}, an upper bound on \cref{eq:ml} can show when the KL divergence is small, which indicates inference has been successful and hyperparameter estimates are likely to have little bias. \citet{titsias_variational_2014} introduced an upper bound that can be computed in $\BigO\left(NM^2\right)$:
\begin{equation}
     \mcLu \!:=\! -\frac{1}{2} \bfy\transpose\left({\bf Q}_n\!\!+\!t\bfI \right)^{\!-\!1}\bfy \!-\! \frac{1}{2}\!\log\left(\detbar{{\bf Q}_n}\right) - \frac{N}{2}\!\!\log 2\pi. \label{eqn:titsiasupper}
\end{equation}
This gives a \emph{data-dependent} upper bound, that can be computed after seeing the data, and for given inducing inputs.

\subsection{Spectral properties of the covariance matrix}

While for small datasets spectral properties of the covariance matrix can be analyzed numerically, we need a different approach for understanding these properties for a typical large dataset.  The \emph{covariance operator,} $\mathcal{K},$ captures the limiting properties of $\Kff$ for large $N$. It is defined by
\begin{equation}
    \mcK g(\bfx') =  \int_{\mathcal{X}} g(\bfx) k(\bfx,\bfx') p(\bfx)d\bfx,
\end{equation}
where $p(\bfx)$ is a probability density from which the inputs are assumed to be drawn. We assume that $\mcK$ is compact, which is the case if $k(\bfx,\bfx')$ is bounded. Under this assumption, the spectral theorem tells us that $\mcK$ has a discrete spectrum. The (finite) sequence of eigenvalues of $\frac{1}{N}\Kff$ converges to the (infinite) sequence of eigenvalues of $\mcK$ \citep{koltchinskii2000random}. 
  Mercer's Theorem \cite{mercer1909} tells us that for continuous kernel functions,
 \begin{equation}
     k(\bfx,\bfx') = \sum_{m=1}^\infty \lambda_m \phi_m(\bfx)\phi_m(\bfx'),
 \end{equation}
 where the $\left(\lambda_m,\phi_m\right)_{i=1}^\infty$ are eigenvalue-eigenfunction pairs of the operator $\mathcal{K},$ with the eigenfunctions orthonormal in $L^2(\mathcal{X})_p.$ Additionally, $\sum_{m=1}^\infty \lambda_m < \infty$.

\subsection{Selecting the number of inducing variables}
Ideally, the number of inducing variables should be selected to make the $KL(Q||\hat{P})$ small. Currently, the most common advice is to increase the number of inducing variables $M$ until the lower bound (\cref{eqn:elbo}) no longer improves. This is a necessary, but not a sufficient condition for the ELBO to be tight and the KL to be small. Taking the upper bound (\cref{eqn:titsiasupper}) into consideration, we can guarantee a good approximation when difference between the upper and lower bounds converges to zero, as this upper bounds the KL.


Both these procedures rely on bounds computed for a given dataset, and a specific setting of variational parameters. While practically useful, they do not tell us how many inducing variables we should expect to use \emph{before} observing any data. In this work, we focus on \emph{a priori} bounds, and asymptotic behavior as $N\to\infty$ and $M$ grows as a function of $N$. These bounds guarantee how the variational method scales computationally for \emph{any} dataset satisfying intuitive conditions. This is particularly important for continual learning scenarios, where we incrementally observe more data. With our a priori results we can guarantee that the growth in required computation will not exceed a certain rate.

\section{Bounds on the KL divergence for eigenfunction inducing features}\label{sec:eigenbounds}
In this section, we prove a priori bounds on the KL divergence using inducing features that rely on spectral information about the covariance matrix or the associated operator. The results in this section form the basis for bounds on the KL divergence for inducing points (\cref{sec:IPandproofs}).

\subsection{A Posteriori Bounds on the KL Divergence}
We first consider  \emph{a posteriori} bounds on the KL divergence that hold for any $\bfy,$ derived by looking at the difference between $\mcLu$ and $\mcLl.$ We will use these bounds in later sections to analyze asymptotic convergence properties. 
\begin{lemma}\label{lem:KLbound}
Let $\tilk=\Kff-\Qff,$  $t=\Tr(\tilk)$ and $\widetilde{\lambda}_{max}$ denote the largest eigenvalue of $\tilk$. Then,
\begin{align*}
	\KL{Q}{\hat{P}} &\!\leq\! \frac{1}{2\sn^2}\!\left(t \!+\! \frac{\widetilde{\lambda}_{max}\|\bfy\|^2_2}{\sn^2\!+\!\widetilde{\lambda}_{max}}\right) \!\leq\! \frac{t}{2\sn^2}\!\left(1\!+\! \frac{\|\bfy\|^2_2}{\sn^2\!+\!t}\right) .
\end{align*}
\end{lemma}

The proof bounds the difference between a refinement of $\mcLu$ also proven by \citet{titsias_variational_2014} and $\mcLl$  through an algebraic manipulation and is given in \cref{app:lemma1}. The second inequality is a consequence of $t \geq \widetilde{\lambda}_{max}.$
We typically expect $\|\vy\|_2^2 = \BigO(N)$, which is the case when the variance of the observed $y$s is bounded, so if $t\ll 1/N$ the KL divergence will be small.

\subsection{A priori bounds: averaging over $\bfy$}\label{sec:avgbound}
\Cref{lem:KLbound} is typically overly pessimistic, as it assumes $\bfy$ can be parallel to the largest eigenvector of $\tilk.$ In this section, we consider a bound that holds \emph{a priori} over the training outputs, when they are drawn from the model. This allows us to bound the KL divergence for a `typical' dataset. 
 
\begin{lemma}\label{lem:avgKLbound}
For any set of $\{\bfx_i\}_{i=1}^N$, if the outputs $\{y_i\}_{i=1}^N$ are generated according to our generative model, then 
\begin{equation}
\frac{t}{2\sn^{2}}\leq  \E_y \left[ \KL{Q}{\hat{P}} \right] \leq \frac{t}{\sn^2}
\end{equation}
\end{lemma}
The lower bound tells us that \emph{even if the training data is contained in an interval of fixed length, we need to use more inducing points for problems with large $N$ if we want to ensure the sparse approximation has converged}. This is shown in \Cref{fig:fixed_m} for data uniformly sampled on the interval $[0,5]$ with 15 inducing points. 
\begin{proof}[Sketch of Proof]\renewcommand{\qedsymbol}{}
\begin{figure}
    \centering
    \includegraphics[width=0.48\textwidth]{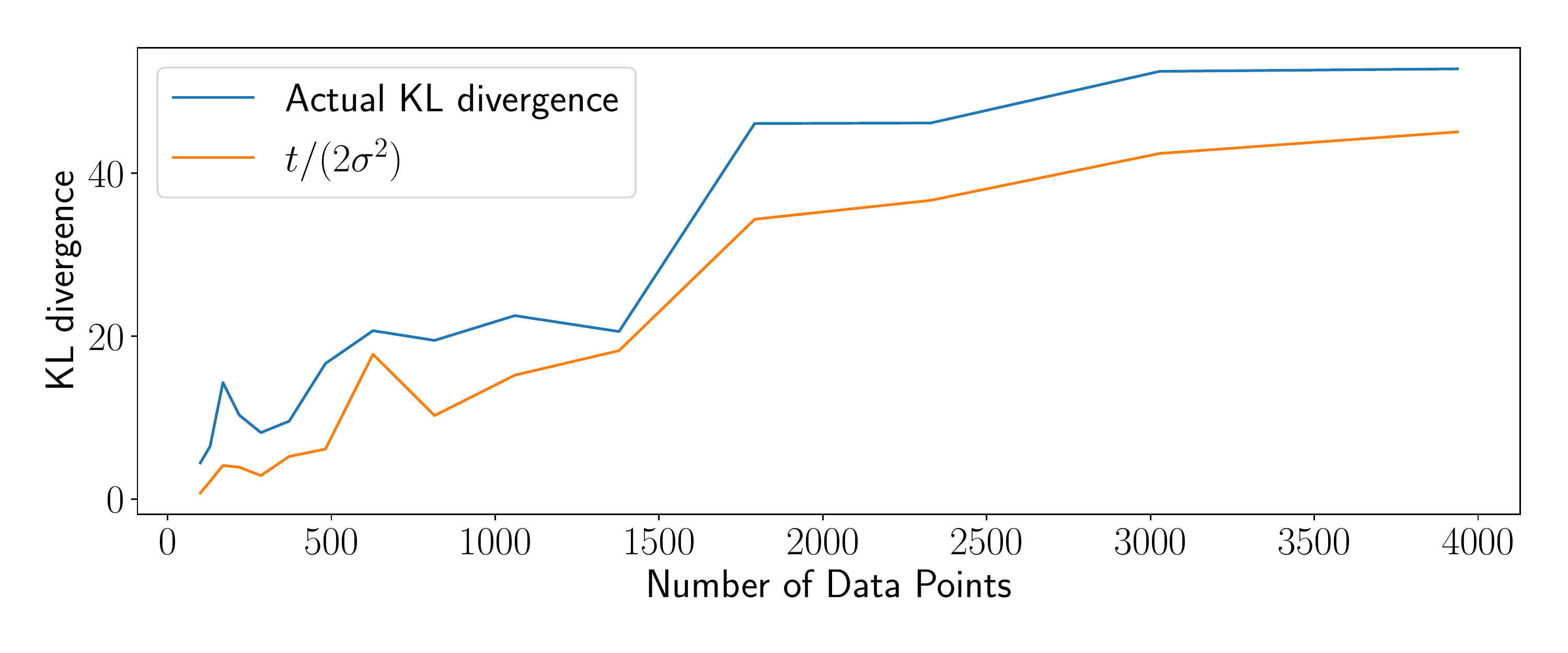} \vspace{-.5cm}
    \caption{Increasing $N$ with fixed $M$ increases the expected KL divergence. $t/2\sn^2$ is a lower bound for the expected value over the KL divergence when $\bfy$ is generated according to our prior model.}
    \label{fig:fixed_m}
\end{figure}
\begin{align*}
        &\Exp{y}{ \KL{Q}{\hat{P}}}  = \frac{t}{2\sn^2}+  \int \NormDist{\bfy; 0, {\bf K}_n} \\ &\qquad\times \log \left( \frac{\NormDist{\bfy; 0,{\bf K}_n}}{\NormDist{\bfy; 0,{\bf Q}_n}}\right)\calcd\bfy
\end{align*}
The second term on the right is a KL divergence between centered Gaussian distributions. The lower bound follows from Jensen's inequality. The proof of the upper bound (\cref{app:KLdiv}) bounds this KL divergence above by $t/(2\sn^2).$ 
\end{proof}

\subsection{Minimizing the upper bound: an idealized case}\label{sec:eigenvec_feats}
We now consider the set of $M$ interdomain inducing features that minimize the upper bounds in \Cref{lem:KLbound,lem:avgKLbound}. Taking into account the lower bound in \Cref{lem:avgKLbound}, they must be within a factor of two of the optimal features defined without reference to training outputs under the assumption of \Cref{lem:avgKLbound}.
Consider
\[ u_m := \sum_{i=1}^N w_i^{(m)}f(\bfx_i)\]
where $w_i^{(m)}$ is the $i^{th}$ entry in the $m^{th}$ eigenvector of $\Kff.$ That is, $u_m$ is a linear combination of inducing points placed at each data point, with weights coming from the entries of the $m^{th}$ eigenvector of $\Kff.$
We show in \cref{app:covcalcs} that
\begin{align}
&\cov(u_m,u_k) =  \mathbf{w}^{(m)\mathrm{\textsf{\tiny T}}}\Kff \mathbf{w}^{(k)}= \lambda_k(\Kff) \delta_{m,k}, \\
&\cov(u_m,f(\bfx_i)) = \left[\Kff \mathbf{w}^{(m)}\right]_i = \lambda_m(\Kff) w^{(m)}_i.
\end{align}

Inference with these features can be seen as the variational equivalent of the optimal parametric projection of the model derived by \citet{ferrari-trecate_finite-dimensional_1999}.

Computation with these features requires computing the matrices $\Kuf$ and $\Kuu.$ $\Kuu$ contains the first $M$ eigenvalues of $\Kff,$ $\Kuf$ contains the corresponding eigenvectors. Computing the first $M$ eigenvalues and vectors (i.e. performing atruncated SVD of $\Kff$) can be done in $\BigO(N^2M)$ using, for example, Lanczos iteration \cite{lanczos1950iteration}. With these features $\Qff$ is the \emph{optimal} rank-$M$ approximation to $\Kff$ and leads to $\tilde{\lambda}_{max}=\lambda_{M+1}(\Kff)$ and $t= \sum_{m=M+1}^N\lambda_{m}(\Kff).$


\subsection{Eigenfunction inducing features}
We now modify the construction given in \cref{sec:eigenvec_feats} to no longer depend on $\Kff$ explicitly (which depends on the specific training inputs) and instead depend on assumptions about the training data. This construction is the \emph{a priori} counterpart of the eigenvector inducing features, as it is defined prior to observing a specific set of training inputs.

Consider the limit as we have observed a large amount of data, so that $\frac{1}{N}\Kff \to \mcK.$ This leads us to replace the eigenvalues, $\{\lambda_m(\Kff)\}_{m=1}^M,$ with the operator eigenvalues, $\{\lambda_m\}_{m=1}^M,$ and the eigenvectors, $\{\mathbf{w^{(m)}}\}_{m=1}^M,$ with the eigenfunctions, $\{\phi_m\}_{m=1}^M,$ yielding
\begin{equation}
    u_m  = \int_{\mathcal{X}}f(\bfx) \phi_m(\bfx)p(\bfx)d\bfx. 
\end{equation}
Note that $p(\bfx)$ influences $u_m$.
In \cref{app:covcalcs}, we show
\[
\cov(u_m,u_k)  \!=\! \lambda_m \delta_{m,k} \text{\, and \,} \cov(u_m,f(\bfx_i)) \!= \!\lambda_m \phi_m(\bfx_i).
\]
These features can be seen as the variational equivalent of methods utilizing truncated priors proposed in \citet{zhu_gaussian_1997}, which are the optimal linear $M$ dimensional parametric GP approximation defined \emph{a priori}, in terms of minimizing expected mean square error. 

In the case of the SE kernel and Gaussian inputs, closed form expressions for eigenfunctions and values are known \citep{zhu_gaussian_1997}. For Mat\'{e}rn kernels with inputs uniform on $[a,b]$, expressions for the eigenfunctions and eigenvalues needed in order to compute $\Kuf$ and $\Kuu$ can be found in \citet{youla1957solution}. However, the formulas involve solving systems of transcendental equations limiting the practical applicability of these features for Mat\'ern kernels.

\subsection{A priori bounds on the KL divergence for eigenfunction features}
Having developed the necessary preliminary results, we now prove the first a priori bounds on the KL divergence. We start with eigenfunction features, which can be implemented practically in certain instances discussed above.
\begin{theorem}\label{thm:boundsforeigfeats}
Suppose $N$ training inputs are drawn i.i.d.~according to input density $p(\bfx).$ For inference with $M$ eigenfunction inducing variables defined with respect to the prior kernel and $p(\bfx),$ with probability at least $1-\delta,$ 
\begin{equation}
	\KL{Q}{\hat{P}} \leq \frac{C}{2\sn^2\delta}\left(1+ \frac{\|\bfy\|^2_2}{\sn^2}\right) 
\end{equation}
where we have defined $C= N\sum_{m=M+1}^\infty \lambda_m,$ and the $\lambda_m$ are the eigenvalues of the integral operator $\mcK$ associated to the prior kernel and $p(\bfx).$
\end{theorem}

\begin{theorem}\label{thm:avgeigencase}
With the assumptions and notation of \Cref{thm:boundsforeigfeats} if $\bfy$ is distributed according to a sample from the prior generative model, with probability at least $1-\delta,$ 
\begin{equation}
	\KL{Q}{\hat{P}} \leq \frac{C}{\delta\sn^2},
\end{equation}
\end{theorem}

\begin{proof}[Sketch of Proof of \Cref{thm:boundsforeigfeats,thm:avgeigencase}]
We first prove a bound on $t$ that holds in expectation over input data matrices of size $N$ with entries drawn i.i.d.~from $p(\bfx).$
 A direct computation of $\Qff$ shows that $\left[\Qff\right]_{i,j} = \sum_{m=1}^{M} \lambda_m \phi_m(\bfx_i)\phi_m(\bfx_j).$ Using the Mercer expansion of the kernel matrix and subtracting,
\[\left[\tilk\right]_{i,i} = \sum_{m=M+1}^\infty \lambda_m\phi^2_m(\bfx_i).\]
Summing this and taking the expectation,
\begin{align*}\label{eq:expected_trace}
\Exp{\bfX}{t}  =  N\sum_{m=M+1}^\infty \lambda_m\Exp{\bfx_i}{\phi^2_m(\bfx_i)} = N\sum_{m=M+1}^\infty \lambda_m.
\end{align*}
The second equality follows from the eigenfunctions having norm 1. Applying Markov's inequality and \Cref{lem:KLbound,lem:avgKLbound} yields \Cref{thm:boundsforeigfeats,thm:avgeigencase} respectively. \qedhere



\end{proof}

\subsection{Squared exponential kernel and Gaussian inputs}\label{sec:segauss1d}
For the SE kernel in one-dimension with hyperparameters $(v,\ell^2)$ and $p(x)\sim \mathcal{N}(0,\sigma^2),$ 
\[
\lambda_m = v\sqrt{2a/A}B^{m-1}
\]
where $a=1/(4\sigma^2),$ $b=1/(2\ell^2),$ $c=\sqrt{a^2+2ab},$ $A=a+b+c$ and $B=b/A$ \citep{zhu_gaussian_1997}. In this case, using the geometric series formula,
\[
 \sum_{m=M+1}^\infty \lambda_m = \frac{v\sqrt{2a}}{(1-B)\sqrt{A}}B^M.
\]

Using this bound with \Cref{thm:boundsforeigfeats,thm:avgeigencase}, we see that by choosing $M=\BigO(\log N),$ under the assumptions of either theorem, we can obtain a bound on the KL divergence that tends to $0$ as $N$ tends to infinity.

\subsection{Mat\'ern kernels and uniform measure}
For the Mat\'ern $k+1/2$ kernel in one dimension, $\lambda_m \asymp m^{-2k-2}$ \citep{rittermulti1995,seeger2008information}, so $\sum_{m=M+1}^\infty \lambda_m = \BigO(M^{-2k-1}).$ In order for the bound in \Cref{thm:avgeigencase} to converge to $0,$ we need $\lim\limits_{N\to \infty} \frac{N}{M^{2k+1}} \to 0.$ This holds if $M=N^{\alpha}$ for $\alpha>\frac{1}{2k+1}.$ For $k>0,$ this bound indicates the number of inducing features can grow sublinearly with the amount of data. 
 
\section{Bounds for inducing points}\label{sec:IPandproofs}

We have shown that using spectral knowledge of either $\Kff$ or $\mcK$ we obtain bounds on the KL divergence indicating that the number of inducing features can be much smaller than the number of data points. While mathematically convenient, the practical applicability of the interdomain features used is limited by computational considerations in the case of the eigenvector features and by the lack of analytic expressions for $\Kuf$ in most cases for the eigenfunction features, as well not knowing the true input density, $p(\bfx)$.

In contrast, inducing points can be efficiently applied to any kernel. In this section, we show that with a good initialization based on the empirical input data distribution, inducing points lead to bounds that are only slightly weaker than the interdomain approaches suggested so far. 

Proving this amounts to obtaining bounds on the trace of the error of a \emph{Nystr\"om approximation} to $\Kff.$ The Nystr\"om approximation, popularized for kernel methods by \citet{williams_using_2001}, approximates a positive semi-definite symmetric matrix by subsampling columns. If $M$ columns, $\{\mathbf{c_i}\}_{i=1}^M,$ are selected from $\Kff,$ the approximation used is $\Kff \approx \mathbf{C}\mathbf{\overline{C}}^{-1}\mathbf{C\transpose},$ where  $\mathbf{C}=  [\mathbf{c_1}, \mathbf{c_2}, \dots, \mathbf{c_M}]$ and $\mathbf{\overline{C}}$ is the $M \times M$ principal submatrix associated to the $\{\mathbf{c_i}\}_{i=1}^M.$ Note that if inducing points are placed at the points associated to each column in the data matrix, then $\Kuu =\mathbf{\overline{C}}$ and $\Kuf\transpose =\mathbf{C},$ so $ \Qff=\mathbf{C}\mathbf{\overline{C}}^{-1}\mathbf{C\transpose}.$

\begin{lemma}{\citep{belabbas_spectral_2009}}\label{lem:detsample}
Given a symmetric positive semidefinite matrix, $\Kff,$ if $M$ columns are selected to form a Nystr\"om approximation such that the probability  of selecting a subset of columns, $Z,$ is proportional to the determinant of the principal submatrix formed by these columns and the matching rows, then,
\begin{equation}\label{eqn:nystrom_bellabas}
    \Exp{Z}{\Tr(\Kff-\Qff)} \leq (M+1) \sum_{m=M+1}^N \lambda_m(\Kff).
\end{equation}
\end{lemma}
\begin{figure}
      \centering      
 \includegraphics[width=0.4\textwidth]{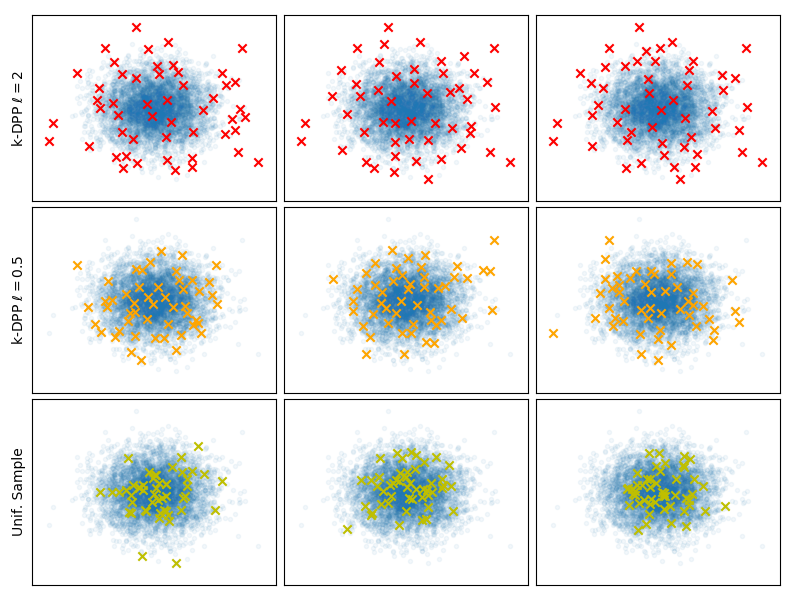}
      \caption{Determinant based sampling, with a SE kernel with $\ell=2$ (top) and with $\ell=.5$ (middle) leads to more dispersed inducing points than uniform sampling (bottom).}\label{fig:det_init}
\end{figure}

This means that on average well-initialized inducing points lead to bounds within a factor of $M+1$ of eigenvector inducing features.

The selection scheme described introduces negative correlations between inducing points locations, leading the $\bfz_i$ to be well-dispersed amongst the training data, as shown in \cref{fig:det_init}. The strength of these negative correlations is determined by the particular kernel.

The proposed initialization scheme is equivalent to sampling $Z$ according to a discrete k-Determinantal Point Process (k-DPP), defined over $\Kff$. \citet{belabbas_spectral_2009} suggested that sampling from this distribution, which has support over $\binom{N}{M}$ subsets of columns, may be computationally infeasible. \Citet{kulesza2011k} provided an exact algorithm for sampling from k-DPPs given an eigendecomposition of the kernel matrix. In our setting, we require our initialisation scheme to have similar computational cost to computing the sparse GP bounds, which prohibits us from performing eigendecomposition. Instead, we rely on cheaper ``$\epsilon$ close'' sampling methods. We therefore provide the following corollary of \cref{lem:detsample}, proven in \cref{app:algorithm}.
\begin{cor}\label{cor:approxkdpp}
Suppose the inducing points, $Z,$ are sampled from an $\epsilon$ k-DPP, $\nu,$ i.e a distribution over subsets of $\bfX$ of size $M$ satisfying, $d(\mu,\nu)_{TV}\leq \epsilon$ where $d(\cdot,\cdot)_{TV}$ denotes total variation distance and $\mu$ is a k-DPP on $\Kff$. Suppose the $k(\bfx,\bfx)<v$ for all $\bfx \in \mathcal{X}.$ Then
\begin{equation}\label{eqn:approx_nystrom_bellabas}
    \Exp{Z\sim \nu}{t} \leq (M+1) \sum_{m=M+1}^N \lambda_m(\Kff) + 2Nv \epsilon.
\end{equation}
\end{cor}

\subsection{A priori bounds for inducing points}

We show analogues of \cref{thm:boundsforeigfeats,thm:avgeigencase} for inducing points.

\begin{theorem}\label{thm:generalexplicit}
Suppose $N$ training inputs are drawn i.i.d according to input density $p(\bfx),$ and $k(\bfx, \bfx)<v$ for all $\bfx \in \mathcal{X}.$ Sample $M$ inducing points from the training data with the probability assigned to any set of size $M$ equal to the probability assigned to the corresponding subset by an $\epsilon$ k-DPP with $k=M$. With probability at least $1-\delta,$ 
\begin{equation}
	\KL{Q}{\hat{P}} \leq \frac{C(M+1)+2Nv\epsilon}{2\sn^2\delta}\left(1+ \frac{\|\bfy\|^2_2}{\sn^2}\right) 
\end{equation}
where $C= N\sum_{m=M+1}^\infty \lambda_m,$ and $\lambda_m$ are the eigenvalues of the integral operator $\mcK$ associated to kernel, $k,$ and $p(\bfx).$
\end{theorem}

\begin{theorem}\label{thm:avgcase}
With the assumptions and notation of \cref{thm:generalexplicit} and if $\bfy$ is distributed according to a sample from the prior generative model, with probability at least $1-\delta,$ 
\begin{equation}
	\KL{Q}{\hat{P}} \leq \frac{C(M+1)+ 2Nv\epsilon}{\delta\sn^2}.
\end{equation}
\end{theorem}
\begin{proof}
We prove \cref{thm:avgcase}. \Cref{thm:generalexplicit} follows the same argument, replacing the expectation over $\bfy$ with the bound given by  \cref{lem:KLbound}. 
\begin{align}
 \Exp{\mathbf{X}}{\Exp{Z|\bfX}{\Exp{\bfy}{\KL{Q}{\hat{P}}}}} &\leq \sn^{-2}\Exp{\mathbf{X}}{\Exp{Z|\bfX}{t}} \nonumber\\
 & \hspace{-3cm}\leq \sn^{-2}(M+1) \E_\mathbf{X}\left[\sum_{m=M+1}^N \lambda_m(\Kff)\right] + 2Nv\epsilon\nonumber\\
 & \hspace{-3cm}\leq \sn^{-2} (M+1)N \sum_{m=M+1}^\infty \lambda_m + 2Nv\epsilon. \nonumber
\end{align}
The first two inequalities use \cref{lem:avgKLbound} and corollary \ref{cor:approxkdpp}. The third follows from noting that the sum inside the expectation is the error in trace norm of the optimal rank $M$ approximation to the covariance matrix for any given $\bfX$, and is bounded above by the error from the rank $M$ approximation due to eigenfunction features. We showed that this error is in expectation equal to $N \sum_{m=M+1}^\infty \lambda_m$ so this must be an upper bound on the expectation in the second to last line.\footnote{ \citet[Proposition 4]{shawe2005eigenspectrum} gives a different proof of the final inequality.}

We apply Markov's inequality, yielding for any $\delta \in (0,1)$ with probability at least $1-\delta,$
\begin{equation*}
    \KL{Q}{\hat{P}}\!\leq\! \frac{(M+1)N \sum_{m=M+1}^\infty \lambda_m+2NV\epsilon}{\delta\sn^{2}}. \qedhere
\end{equation*}
\end{proof}
\Cref{fig:kl_divergence_bound} compares the actual KL divergence, the \emph{a posteriori} bound derived by $\mcLu-\mcLl,$ and the bounds proven in theorems \ref{thm:generalexplicit} and \ref{thm:avgcase} on a dataset with normally distributed training inputs and $\bfy$ drawn from the generative model.
\section{Consequences of \cref{thm:generalexplicit} and \cref{thm:avgcase}}
We now investigate implications of our main results for sparse GP regression. Our first two corollaries consider Gaussian inputs and the squared exponential (SE) kernel, and show that in $D$ dimensions, choosing $M=\BigO(\log^D(N))$ is sufficient in order for the KL divergence to converge with high probability. We then briefly summarize convergence rates for other stationary kernels. Finally we point out consequences of our definition of convergence for the quality of the pointwise posterior mean and uncertainty. 


\subsection{Comparison of consequences of theorems}
Using the explicit formula for the eigenvalues given in \cref{sec:segauss1d}, we arrive at the following corollary:
\begin{cor}\label{cor:gaussklM}
Suppose that $\|\bfy\|_2^2 \leq  RN.$ Fix $\gamma>0,$ and take $\epsilon=\frac{\delta\sn^2}{vN^{\gamma+2}}.$ Assume the input data is normally distributed and regression in performed with a SE kernel. Under the assumptions of \cref{thm:generalexplicit},  with probability $1-\delta,$ %
\begin{align}
	\KL{Q}{\hat{P}} \leq N^{-\gamma}\left(2R/\sn^2 + 2/N\right). 
\end{align}

when inference is performed with $M=\frac{(3+\gamma)\log(N)+\log D}{\log(B^{-1})},$ where $D =\frac{v\sqrt{2a}}{2\sqrt{A}\sn^2\delta(1-B)}.$ 
\end{cor}
The proof is given in \cref{app:cors}. If the lengthscale is much shorter than the standard deviation of the data then $B$ will be near 1, implying that $M$ will need to be large in order for the bound to converge.
\begin{rem}
The assumption $\|\bfy\|_2^2 \leq RN$ for some $R$ is very weak. For example, if $\bfy$ is a realization of an integrable function with constant noise, 
\[
\sum_{i=1}^N y_i^2 \leq \sum_{i=1}^N f(\bfx_i)^2 +\sum_{i=1}^N \epsilon_i^2  + o(N)
\]
The first sum is asymptotically $N\int f(\bfx)p(\bfx)dx,$ and the second is asymptotically $N\sn^2.$
\end{rem}
The consequence of corollary \ref{cor:gaussklM} is shown in \cref{fig:increasing_n}, in which we gradually increase $N,$ choosing $M=C\log(N)+C_0,$ and see the KL divergence converges as an inverse power of $N.$ The training outputs are generated from a sample from the prior generative model. Note that as \cref{thm:avgcase} assumes $\bfy$ is sampled from the prior and is not derived using the upper bound (\cref{eqn:titsiasupper}); it may be tighter than the a posteriori bound in cases when this upper bound is not tight.

For the SE kernel and Gaussian inputs, the rate that we prove $M$ must increase for inducing points and eigenfunction features differs by a constant factor. For the Mat\'ern $k+1/2$ kernel in one dimension, we need to choose $M=N^{\alpha}$ with $\alpha > 1/(2k)$ instead of $\alpha > 1/(2k+1).$ This difference is particularly stark in the case of the Mat\'ern $3/2$ kernel, for which our bounds tell us that inference with inducing points requires $\alpha>1/2$ as opposed to $\alpha>1/3$ for the eigenfunction features. Whether this is an artifact of the proof, the initialization scheme, or an inherent limitation for inducing points is an interesting area for future work.

\begin{figure}
    \centering
    \includegraphics[width=0.46\textwidth]{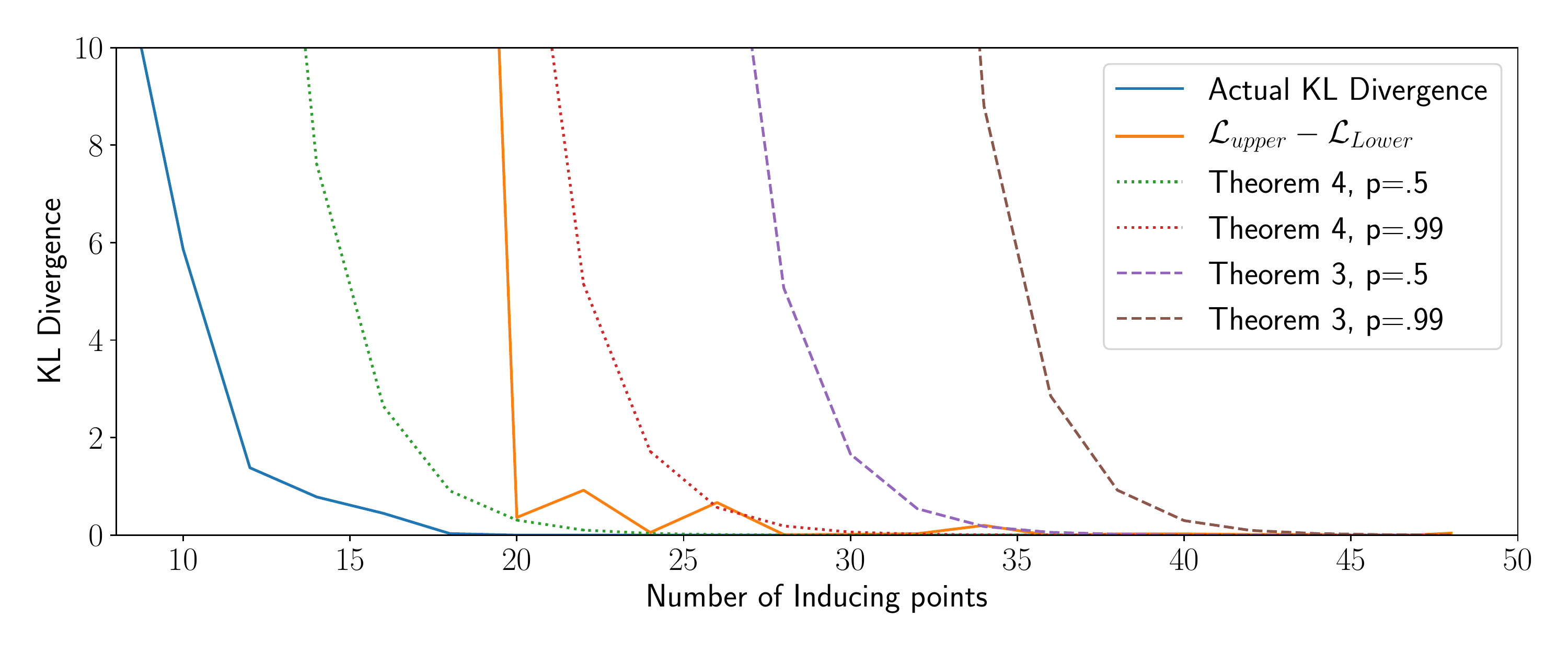}\vspace{-.5cm}
    \caption{Rates of convergence as $M$ increases on fixed dataset of size $N=1000$, with a SE-kernel with $\ell=.6, v=1, \sn =1$ and $x \sim \mathcal{N}(0,1)$ and $\bfy$ sampled from the prior.}
    \label{fig:kl_divergence_bound}
\end{figure}
\begin{figure}
    \centering
    \includegraphics[width=0.44\textwidth]{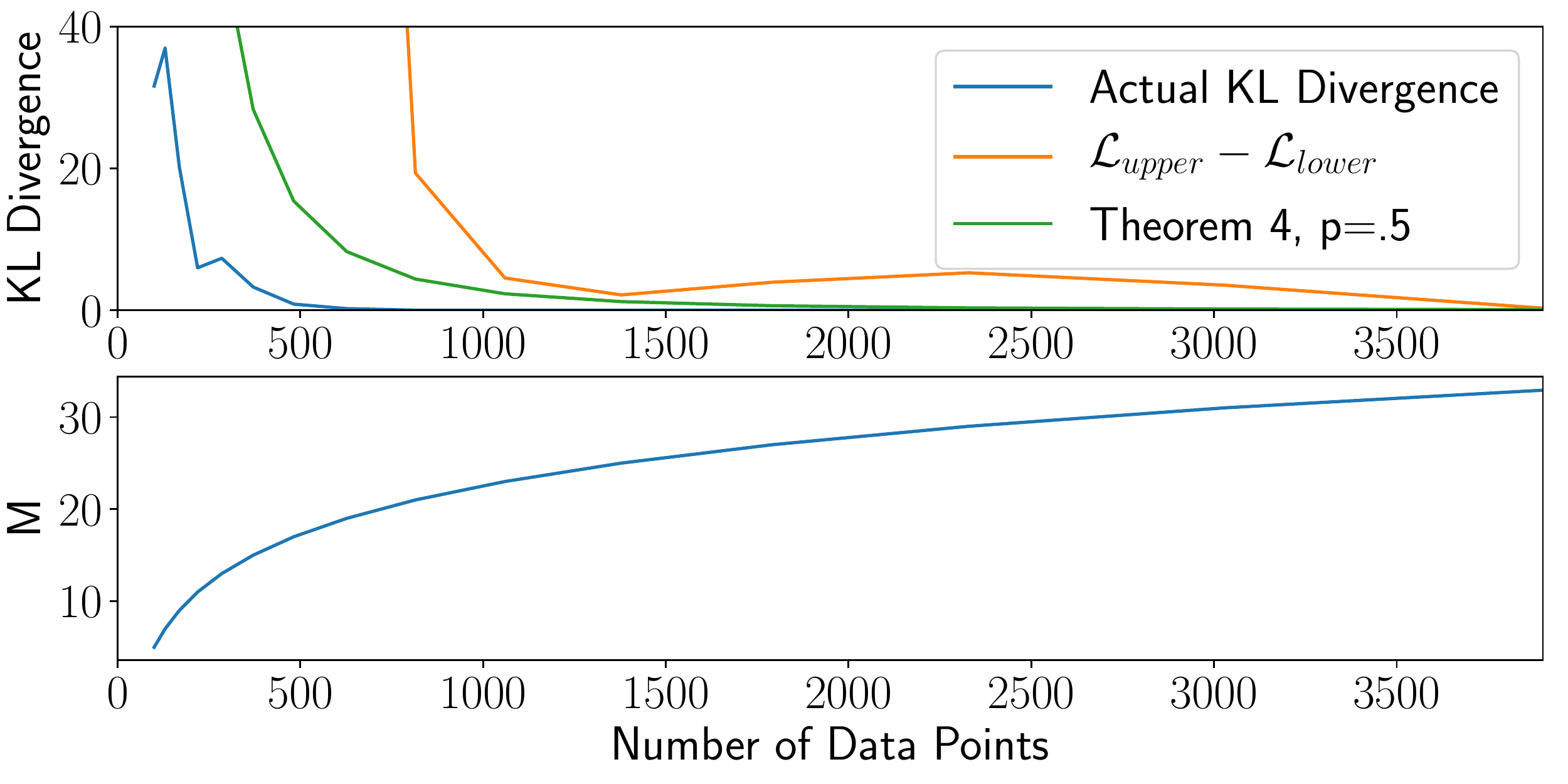}\vspace{-.5cm}
    \caption{We increase $N$ and take $M = C\log(N)$ for a one-dimensional SE-kernel and normally distributed inputs. The KL divergence decays rapidly, as predicted by Corollary \ref{cor:gaussklM}.}
    \label{fig:increasing_n}
\end{figure}

\subsection{Multidimensional data, effect of input density and other kernels}
If $\mathcal{X}=\R^D,$ it is common to choose a \emph{separable kernel}, i.e a kernel that can be written as a product of kernels along each dimension. If this choice is made, and input densities factor over the dimensions, the eigenvalues of $\mathcal{K}$ are the product of the eigenvalues along each dimension. 
In the case of the SE-ARD kernel and Gaussian input distribution, we obtain an analogous statement to corollary \ref{cor:gaussklM} in D-dimensions. 
\begin{cor}\label{cor:gaussMD}
For any fixed $\epsilon',\delta>0$ under the assumptions of corollary \ref{cor:gaussklM}, with a SE-ARD kernel in $D$ dimensions, $p(\bfx)$ a multivariate Gaussian, $M=\BigO(\log^D N)$ inducing points and $\epsilon=\BigO(N^{-\gamma})$ for some fixed $\gamma>2,$ with probability at least $1-\delta, \, \, \KL{Q}{\hat{P}} \leq \delta^{-1}\epsilon'.$

\end{cor}

\begin{table*}[ht]
\caption{The number of features needed for our bounds to converge in $D-$dimensions. These hold for some fixed $\alpha>0$ and any $\epsilon'>0$.}
\vspace{-0.25cm}
\label{table:numfeats}
\vskip 0.15in
\begin{center}
\begin{small}
\begin{sc}
\begin{tabular}{lcccr}
\toprule
Kernel & Input Distribution & Decay of $\lambda_m$ & M, Theorem \ref{thm:generalexplicit} & M, Theorem \ref{thm:avgcase}  \\
\midrule
SE-Kernel   & Compact Support &  $\BigO\left(\exp(-\alpha\frac{m}{d}\log\frac{m}{d})\right)$ &$\BigO(\log^D(N))$ & $\BigO(\log^D(N))$\\
SE-Kernel &  Gaussian &  $\BigO\left(\exp(-\alpha\frac{m}{d}\right))$  &$\BigO(\log^D(N))$ & $\BigO(\log^D(N))$\\
Mat\'{e}rn $k$+1/2\footnotemark   & Uniform on Interval & $\BigO\left(M^{-2k-2}\log(M)^{2(d-1)(k+1)}\right)$  & $\BigO\left(N^{1/k+\epsilon'}\right)$ & $\BigO\left(N^{1/(2k)+\epsilon'}\right)$\\
\bottomrule
\end{tabular}
\end{sc}
\end{small}
\end{center}
\vskip -0.1in
\vspace{-0.25cm}
\end{table*}

The proof uses ideas from \citet{seeger2008information} and is given in \cref{app:cors}. While for the SE kernel and Gaussian input density $M$ can grow polylogarithmically in $N,$ and the KL divergence still converges, this is not the case for regression with other kernels or input distribution.

Closed form expressions for the eigenvalues of operators associated to many kernels and input distributions are not known. For stationary kernels and compactly supported input distributions, the asymptotic rate of decay of the eigenvalues of $\mcK$ is well-understood \cite{widom_asymptotic_1963,widom1964asymptotic,rittermulti1995}. The intuitive summary of these results is that smooth kernels, with concentrated input distributions have rapidly decaying eigenvalues. In contrast, kernels such as the Mat\'ern-1/2 that define processes that are not smooth have slowly decaying eigenvalues. For Lebesgue measure on $[a,b]$ the Sacks-Ylivasker conditions of order $r$ (\cref{app:SY}), which can be roughly thought of as meaning that realizations of the process are $r$ times differentiable with probability 1 \citep{rittermulti1995}, implies an eigendecay of $\lambda_m \asymp m^{-2r-2}.$ \Cref{table:numfeats} summarizes the spectral decay of several stationary kernels, as well as the implications for the number of inducing points needed for inference to provably converge with our bounds.

\subsection{Computational complexity}

We now have all the components necessary for analyzing the overall computational complexity of finding an arbitrarily good GP approximation.
To understand the full computational complexity, we must consider the cost of initializing the inducing points using an exact or approximate k-DPP, as well as the $\BigO(NM^2)$ time complexity of variational inference. Recent work of \citet{exactkdpp} indicate that an exact algorithm for sampling a k-DPP can be implemented in $\BigO(N\log(N)\text{poly}(M)).$ We base our method on \citet{pmlr-v49-anari16}, who show that an $\epsilon$ k-DPP can be sampled via MCMC methods in $\BigO(NM^4\log(N)\!+\!NM^3\log(\frac{1}{\epsilon}))$ time with memory $\BigO(N\!+\!M^2)$ (see \cref{app:algorithm}).\footnotetext{The kernel discussed here is the product of 1-dimensional Mat\'{e}rn kernels with the same smoothness parameter, which differs from the standard definition if $D>1$. Bounds on the eigenvalues of the standard multidimensional Mat\'ern kernel appear in \citet{seeger2008information}.}\footnote{Open source implementations of approximate k-DPPs are available (e.g. \cite{GaBaVa18}).} We can choose $\epsilon$ to be any inverse power of $N$ which only adds a constant factor to the complexity. For the SE kernel, taking $M\!=\!\BigO(\log^D\!N)$ inducing points leads to a complexity of $\BigO(N\log^{4D+1}N),$ a large computational saving compared to the $\BigO(N^3)$ cost of exact inference. For the Mat\'ern $k\!+\!\frac{1}{2}$ kernel in one-dimension and the average case analysis of \cref{thm:avgcase} we need to take $M\!=\!\BigO(N^{1/(2k)+\epsilon'})$ implying a computational complexity of $\BigO(N^{1+2/k+4\epsilon'}\log(N))$ which is an improvement over the cost of full inference for $k\!>\!1.$ Improvements in methods for sampling exact or approximate k-DPPs (e.g. recent bounds on mixing time \citep{Hermon2019ModifiedLI}) or bounds on other selection schemes for Nystr\"om approximations directly translate to improved bounds on the computational cost of convergent sparse Gaussian process approximations through this framework. 

\subsection{Pointwise approximate posterior}
In many applications, pointwise estimates of the posterior mean and variance are of interest. It is therefore desirable that the approximate variational posterior gives similar estimates of these quantities as the true posterior. \citet{Huggins2018ScalableGP} derived an approximation method for sparse GP inference with provable guarantees about pointwise mean and variance estimates of the posterior process and showed that approximations with moderate KL divergences can still have large deviations in mean and variance estimates. However, if the KL divergence converges to zero, estimates of mean and variance converge to the posterior values. By the chain rule of KL divergence \cite{matthews_sparse_2016},
\begin{align*}
&\KL{\mu_{\mathcal{X}}}{\nu_{ \mathcal{X}}} = \KL{\mu_{\bfx_{*}}}{\nu_{\bfx_{*}}} \\&+ \Exp{\mu_{\bfx_{*}}}{ \KL{\mu_{\mathcal{X}\backslash\bfx_{*}|\bfx_{*}}}{ \nu_{\mathcal{X}\backslash\bfx_{*}|\bfx_{*}}}} \geq \KL{\mu_{\bfx_{*}}}{\nu_{\bfx_{*}}}.
\end{align*}
 Therefore, bounds on the mean and variance of a one-dimensional Gaussian with a small KL divergence imply pointwise guarantees about posterior inference when the KL divergence between processes is small.

\begin{prop}\label{prop:pointwise}
Suppose $q$ and $p$ are one dimensional Gaussian distributions with means $\mu_1$ and $\mu_2$ and variances $\sigma_1$ and $\sigma_2,$ such that $2\KL{q}{p}= \varepsilon \leq \frac{1}{5},$ then 
\begin{align*}
\lvert\mu_1-\mu_2\rvert \leq \sigma_2\sqrt{\varepsilon} \leq \frac{\sigma_1\sqrt{\varepsilon}}{\sqrt{1-\sqrt{3\varepsilon}}}
\text{ and }
\lvert 1- \sigma_1^2/\sigma_2^2 \rvert<\sqrt{3\varepsilon}.
\end{align*}
\end{prop}

The proof is in \cref{app:KLdiv}. If $\epsilon \to 0,$ proposition \ref{prop:pointwise} implies $\mu_1 \to \mu_2$ and $\sigma_1 \to \sigma_2.$ Using this and  \cref{thm:generalexplicit,thm:avgcase}, \emph{the posterior mean and variance converge pointwise to those of the full model using $M\ll N$ inducing features}.

\section{Related work}
 Statistical guarantees for convergence of parametric GP approximations \citep{zhu_gaussian_1997, ferrari-trecate_finite-dimensional_1999}, lead to similar conclusions about the choice of approximating rank. \citet{ferrari-trecate_finite-dimensional_1999} showed that given $N$ data points, using a rank $M$ truncated SVD of the prior covariance matrix, such that $\lambda_M \ll \sn^2/N$ results in almost no change in the model, in terms of expected mean squared error. Our results can be considered the equivalent for variational inference, showing that theoretical guarantees can be established for \emph{non-parametric} approximate inference. 


 Guarantees on the quality Nystr\"om approximations have been used to bound the error of approximate kernel methods, notably for kernel ridge regression \cite{AlaouiFast2015,li2016fast}. The specific method for selecting columns in the Nystr\"om approximation plays a large role in these analyses. \Citet{li2016fast} use an approximate k-DPP, nearly identical to the initialization we analyze; \citet{AlaouiFast2015} sample columns according to ridge leverage scores. The substantial literature on bounds for Nystr\"om approximations \citep[e.g.][]{gittens_revisiting_2013} motivates considering other initialization schemes for inducing points in the context of Gaussian process regression.
\section{Conclusion}

We proved bounds on the KL divergence between the variational approximation of sparse GP regression to the posterior, that depend only on the decay of the eigenvalues of the covariance operator. These bounds prove the intuitive result that \emph{smooth kernels with training data concentrated in a small region admit high quality, very sparse approximations}. These bounds prove that \emph{truly sparse non-parametric inference, with $M\ll N,$} can provide reliable estimates of the marginal likelihood and pointwise posterior. 

Extensions to models with non-conjugate likelihoods, especially within the framework of \citet{hensman_scalable_2015}, pose a promising direction for future research.
\section*{Acknowledgements}

We would like to thank James Hensman for providing an excellent research environment, and the reviewers for their helpful feedback and suggestions.  
We would also particularly like to thank Guillaume Gautier, for pointing out an error in the exact k-DPP sampling algorithm cited in an earlier version of this work, and for guiding us through recent work on sampling k-DPPs.

\bibliography{icml2019IP}

\begin{thebibliography}{36}
\providecommand{\natexlab}[1]{#1}
\providecommand{\url}[1]{\texttt{#1}}
\expandafter\ifx\csname urlstyle\endcsname\relax
  \providecommand{\doi}[1]{doi: #1}\else
  \providecommand{\doi}{doi: \begingroup \urlstyle{rm}\Url}\fi

\bibitem[Alaoui \& Mahoney(2015)Alaoui and Mahoney]{AlaouiFast2015}
Alaoui, A. and Mahoney, M.~W.
\newblock {Fast Randomized Kernel Ridge Regression with Statistical
  Guarantees}.
\newblock In \emph{Advances in Neural Information Processing Systems 28
  (NIPS)}, pp.\  775--783. 2015.

\bibitem[Anari et~al.(2016)Anari, Gharan, and Rezaei]{pmlr-v49-anari16}
Anari, N., Gharan, S.~O., and Rezaei, A.
\newblock {Monte Carlo Markov Chain Algorithms for Sampling Strongly Rayleigh
  Distributions and Determinantal Point Processes}.
\newblock In \emph{{29th Annual Conference on Learning Theory (COLT)}}, pp.\
  103--115, 2016.

\bibitem[Bauer et~al.(2016)Bauer, van~der Wilk, and
  Rasmussen]{bauer_understanding_2016}
Bauer, M., van~der Wilk, M., and Rasmussen, C.~E.
\newblock {Understanding probabilistic sparse {Gaussian} process
  approximations}.
\newblock In \emph{Advances in Neural Information Processing Systems (NIPS)},
  pp.\  1533--1541, 2016.

\bibitem[Belabbas \& Wolfe(2009)Belabbas and Wolfe]{belabbas_spectral_2009}
Belabbas, M.-A. and Wolfe, P.~J.
\newblock {Spectral Methods in Machine Learning and new Strategies for very
  Large Datasets}.
\newblock In \emph{Proceedings of the National Academy of Sciences (PNAS)},
  volume 106, pp.\  369--374, 2009.

\bibitem[Derezi\`nski et~al.(2019)Derezi\`nski, Calandriello, and
  Valko]{exactkdpp}
Derezi\`nski, M., Calandriello, D., and Valko, M.
\newblock Exact sampling of determinantal point processes with sublinear time
  preprocessing.
\newblock In \emph{{International} {Conference} on {Machine} {Learning},
  {Workshop on Negative Dependence in Machine Learning}}, 2019.

\bibitem[Ferrari-Trecate et~al.(1999)Ferrari-Trecate, Williams, and
  Opper]{ferrari-trecate_finite-dimensional_1999}
Ferrari-Trecate, G., Williams, C.~K., and Opper, M.
\newblock {Finite-dimensional Approximation of {Gaussian} Processes}.
\newblock In \emph{Advances in Neural Information Processing Systems (NIPS)},
  pp.\  218--224, 1999.

\bibitem[Gautier et~al.(2018)Gautier, Bardenet, and Valko]{GaBaVa18}
Gautier, G., Bardenet, R., and Valko, M.
\newblock {DPPy: Sampling Determinantal Point Processes with Python}.
\newblock \emph{ArXiv e-prints}, 2018.
\newblock URL \url{http://arxiv.org/abs/1809.07258}.
\newblock Code at http://github.com/guilgautier/DPPy/ Documentation at
  http://dppy.readthedocs.io/.

\bibitem[Gittens \& Mahoney(2013)Gittens and Mahoney]{gittens_revisiting_2013}
Gittens, A. and Mahoney, M.
\newblock Revisiting the {Nystr\"om} method for improved large-scale machine
  learning.
\newblock In \emph{Proceedings of the 30th {International} {Conference} on
  {Machine} {Learning}}, volume~28 of \emph{Proceedings of {Machine} {Learning}
  {Research}}, pp.\  567--575, Atlanta, Georgia, USA, June 2013. PMLR.

\bibitem[Gradshteyn \& Ryzhik(2014)Gradshteyn and Ryzhik]{gradshteyn2014table}
Gradshteyn, I.~S. and Ryzhik, I.~M.
\newblock \emph{Table of Integrals, Series, and Products}.
\newblock Academic press, 2014.

\bibitem[Hensman et~al.(2013)Hensman, Fusi, and
  Lawrence]{hensman_gaussian_2013}
Hensman, J., Fusi, N., and Lawrence, N.~D.
\newblock Gaussian processes for big data.
\newblock In \emph{Uncertainty in Artificial Intelligence (UAI)}, pp.\  282,
  2013.

\bibitem[Hensman et~al.(2015)Hensman, Matthews, and
  Ghahramani]{hensman_scalable_2015}
Hensman, J., Matthews, A., and Ghahramani, Z.
\newblock {Scalable Variational {G}aussian Process Classification}.
\newblock In \emph{Artificial Intelligence and Statistics (AISTATS)}, pp.\
  351--360, 2015.

\bibitem[Hensman et~al.(2018)Hensman, Durrande, and
  Solin]{hensman2018variational}
Hensman, J., Durrande, N., and Solin, A.
\newblock {Variational {F}ourier Features for {G}aussian Processes}.
\newblock In \emph{Journal of Machine Learning Research}, volume~18, pp.\
  1--52, 2018.

\bibitem[Hermon \& Salez(2019)Hermon and Salez]{Hermon2019ModifiedLI}
Hermon, J. and Salez, J.
\newblock {Modified log-Sobolev inequalities for strong-Rayleigh measures}.
\newblock 2019.

\bibitem[Huggins et~al.(2019)Huggins, Campbell, Kasprzak, and
  Broderick]{Huggins2018ScalableGP}
Huggins, J.~H., Campbell, T., Kasprzak, M., and Broderick, T.
\newblock {Scalable Gaussian Process Inference with Finite-data Mean and
  Variance Guarantees}.
\newblock In \emph{International Conference on Artificial Intelligence and
  Statistics (AISTATS)}, 2019.

\bibitem[Koltchinskii \& Gin{\'e}(2000)Koltchinskii and
  Gin{\'e}]{koltchinskii2000random}
Koltchinskii, V. and Gin{\'e}, E.
\newblock {Random Matrix Approximation of Spectra of Integral Operators}.
\newblock In \emph{Bernoulli}, volume~6, pp.\  113--167, 2000.

\bibitem[Kulesza \& Taskar(2011)Kulesza and Taskar]{kulesza2011k}
Kulesza, A. and Taskar, B.
\newblock {k-DPPs: Fixed-size Determinantal Point Processes}.
\newblock In \emph{Proceedings of the 28th {International Conference on Machine
  Learning} (ICML)}, pp.\  1193--1200, 2011.

\bibitem[Lanczos(1950)]{lanczos1950iteration}
Lanczos, C.
\newblock {An Iteration Method for the Solution of the Eigenvalue Problem of
  Linear Differential and Integral Operators}.
\newblock In \emph{Journal of Research of the National Bureau of Standards},
  pp.\  255--282, 1950.

\bibitem[L\'azaro-Gredilla \& Figueiras-Vidal(2009)L\'azaro-Gredilla and
  Figueiras-Vidal]{lazaro-gredilla_inter-domain_2009}
L\'azaro-Gredilla, M. and Figueiras-Vidal, A.
\newblock Inter-domain {Gaussian} {Processes} for {Sparse} {Inference} using
  {Inducing} {Features}.
\newblock In \emph{Advances in {Neural} {Information} {Processing} {Systems}
  (NIPS) 22}, pp.\  1087--1095. 2009.

\bibitem[Li et~al.(2016)Li, Jegelka, and Sra]{li2016fast}
Li, C., Jegelka, S., and Sra, S.
\newblock {Fast DPP Sampling for Nystrom with Application to Kernel Methods}.
\newblock In \emph{International Conference on Machine Learning}, pp.\
  2061--2070, 2016.

\bibitem[Matthews et~al.(2016)Matthews, Hensman, Turner, and
  Ghahramani]{matthews_sparse_2016}
Matthews, A. G. d.~G., Hensman, J., Turner, R., and Ghahramani, Z.
\newblock {On Sparse Variational Methods and the {K}ullback-{L}eibler
  Divergence between Stochastic Processes}.
\newblock In \emph{Artificial Intelligence and Statistics (AISTATS)}, pp.\
  231--239, 2016.

\bibitem[Mercer(1909)]{mercer1909}
Mercer, J.
\newblock {Functions of Positive and Negative Type, and their Connection the
  Theory of Integral Equations}.
\newblock In \emph{Phil. Trans. R. Soc. Lond. A}, volume 209, pp.\  415--446.
  The Royal Society, 1909.

\bibitem[Neal(1996)]{neal1996bayesian}
Neal, R.~M.
\newblock \emph{{Bayesian Learning for Neural Networks}}, volume 118.
\newblock Springer, 1996.

\bibitem[Qui\~nonero Candela \& Rasmussen(2005)Qui\~nonero Candela and
  Rasmussen]{quin2005unifying}
Qui\~nonero Candela, J. and Rasmussen, C.~E.
\newblock {A Unifying View of Sparse Approximate {Gaussian} Process
  Regression}.
\newblock In \emph{Journal of Machine Learning Research}, volume~6, pp.\
  1939--1959, 2005.

\bibitem[Rahimi \& Recht(2008)Rahimi and Recht]{rahimi_random_2008}
Rahimi, A. and Recht, B.
\newblock {Random Features for Large-scale Kernel Machines}.
\newblock In \emph{Advances in Neural Information Processing Systems (NIPS)},
  pp.\  1177--1184, 2008.

\bibitem[Rasmussen \& Williams(2006)Rasmussen and
  Williams]{rasmussen_gaussian_2005}
Rasmussen, C.~E. and Williams, C.~K.
\newblock \emph{Gaussian Processes for Machine Learning}.
\newblock MIT Press, 2006.

\bibitem[Ritter et~al.(1995)Ritter, Wasilkowski, and
  Wo\'{z}niakowski]{rittermulti1995}
Ritter, K., Wasilkowski, G.~W., and Wo\'{z}niakowski, H.
\newblock {Multivariate Integration and Approximation for Random Fields
  Satisfying {Sacks-Ylvisaker} Conditions}.
\newblock In \emph{The Annals of Applied Probability}, volume~5, pp.\
  518--540. Institute of Mathematical Statistics, 1995.

\bibitem[Seeger et~al.(2008)Seeger, Kakade, and Foster]{seeger2008information}
Seeger, M.~W., Kakade, S.~M., and Foster, D.~P.
\newblock {Information Consistency of Nonparametric {Gaussian} Process
  Methods}.
\newblock In \emph{IEEE Transactions on Information Theory}, volume~54, pp.\
  2376--2382. IEEE, 2008.

\bibitem[Shawe-Taylor et~al.(2005)Shawe-Taylor, Williams, Cristianini, and
  Kandola]{shawe2005eigenspectrum}
Shawe-Taylor, J., Williams, C.~K., Cristianini, N., and Kandola, J.
\newblock On the eigenspectrum of the {Gram} matrix and the generalization
  error of kernel-{PCA}.
\newblock \emph{IEEE Transactions on Information Theory}, 51\penalty0
  (7):\penalty0 2510--2522, 2005.

\bibitem[Titsias(2009)]{titsias_variational_2009}
Titsias, M.~K.
\newblock {Variational Learning of Inducing Variables in Sparse {Gaussian}
  Processes}.
\newblock In \emph{Artificial {Intelligence} and {Statistics} (AISTATS)}, pp.\
  567--574, 2009.

\bibitem[Titsias(2014)]{titsias_variational_2014}
Titsias, M.~K.
\newblock Variational {Inference} for {Gaussian} and {Determinantal} {Point}
  {Processes}.
\newblock In \emph{Workshop on Advances in Variational Inference (NIPS)},
  December 2014.

\bibitem[Turner \& Sahani(2011)Turner and Sahani]{turner_sahani_2011}
Turner, R.~E. and Sahani, M.
\newblock \emph{{Two Problems with Variational Expectation Maximisation for
  Time Series Models}}, pp.\  104--124.
\newblock Cambridge University Press, 2011.

\bibitem[Widom(1963)]{widom_asymptotic_1963}
Widom, H.
\newblock {Asymptotic Behavior of the Eigenvalues of Certain Integral
  Equations. I}.
\newblock In \emph{Transactions of the American Mathematical Society}, volume
  109, pp.\  278--295. American Mathematical Society, 1963.

\bibitem[Widom(1964)]{widom1964asymptotic}
Widom, H.
\newblock {Asymptotic behavior of the eigenvalues of certain integral
  equations. II}.
\newblock In \emph{Archive for Rational Mechanics and Analysis}, volume~17,
  pp.\  215--229. Springer, 1964.

\bibitem[Williams \& Seeger(2001)Williams and Seeger]{williams_using_2001}
Williams, C. K.~I. and Seeger, M.
\newblock Using the {Nystr\"om} {Method} to {Speed} {Up} {Kernel} {Machines}.
\newblock In \emph{Advances in {Neural} {Information} {Processing} {Systems}
  (NIPS) 13}, pp.\  682--688. MIT Press, 2001.

\bibitem[Youla(1957)]{youla1957solution}
Youla, D.
\newblock {The Solution of a Homogeneous {W}iener-{H}opf Integral Equation
  Occurring in the Expansion of Second-order Stationary Random Functions}.
\newblock In \emph{IRE Transactions on Information Theory}, volume~3, pp.\
  187--193. IEEE, 1957.

\bibitem[Zhu et~al.(1997)Zhu, Williams, Rohwer, and
  Morciniec]{zhu_gaussian_1997}
Zhu, H., Williams, C. K.~I., Rohwer, R., and Morciniec, M.
\newblock {Gaussian Regression and Optimal Finite Dimensional Linear Models}.
\newblock In \emph{Neural {Networks} and {Machine} {Learning}}, pp.\  167--184.
  Springer-Verlag, 1997.

\end{thebibliography}
\bibliographystyle{icml2019}

\appendix
\section{Proof Of \Cref{lem:KLbound}}\label{app:lemma1}
\citet{titsias_variational_2014} proves the tighter upper bound,
\begin{multline*}
\mcL \leq \mcLu' := - \frac{N}{2}\log(2\pi)  \\- \frac{1}{2}\log\left(\detbar{{\bf Q}_n}\right) - \frac{1}{2} \bfy\transpose\left({\bf Q}_n+ \widetilde{\lambda}_{max}\bfI \right)^{-1}\bfy.
\end{multline*}
Subtracting,
\begin{align}\label{eqn:Ldif}
&\mcLu' - \mcLl \nonumber\\
&=\frac{t}{2\sn^2} + \frac{1}{2}\left(\bfy\transpose \left({\bf Q}_n^{-1}- ({\bf Q}_n+\widetilde{\lambda}_{max}\bfI)^{-1}\right) \bfy\right).
\end{align}
Since $\Qff$ is symmetric positive semidefinite, ${\bf Q}_n$ is positive definite with eigenvalues bounded below by $\sn^2$. Write, ${\bf Q}_n= \bfU \bfLam \bfU\transpose,$ where $\bfU$ is unitary and $\bfLam$ is a diagonal matrix with non-increasing diagonal entries $\gamma_1\geq \gamma_2 \geq \ldots \geq \gamma_N\geq \sn^2.$ 

We can rewrite the second term (ignoring the factor of one half) in \Cref{eqn:Ldif} as,
\[
(\bfU\transpose \bfy)\transpose \left(\bfLam^{-1}- (\bfLam+\widetilde{\lambda}_{max}\bfI)^{-1}\right)(\bfU\transpose \bfy).
\]
Define, $\bfz=(\bfU\transpose \bfy).$ Since $\bfU$ is unitary, $\|\bfz\|= \|\bfy\|.$
\begin{align*}
(\bfU\transpose \bfy)\transpose \left(\bfLam^{-1}-\right.&\left. (\bfLam+t\bfI)^{-1}\right) (\bfU\transpose \bfy) \\ &=\bfz\transpose\left(\bfLam^{-1}- (\bfLam+\widetilde{\lambda}_{max}\bfI)^{-1}\right) \bfz \nonumber\\
&= \sum_i z_i^2\frac{\widetilde{\lambda}_{max}}{\gamma_i^2+\gamma_i\widetilde{\lambda}_{max}} \nonumber \\
 &\leq \|\bfy\|^2 \frac{\widetilde{\lambda}_{max}}{\gamma_N^2+\gamma_N\widetilde{\lambda}_{max}}.    
\end{align*}

The last inequality comes from noting that the fraction in the sum attains a maximum when $\gamma_i$ is minimized. Since $\sn^2$ is a lower bound on the smallest eigenvalue of ${\bf Q}_n,$ we have,
\[
\bfy^{T} \left({\bf Q}_n^{-1}- ({\bf Q}_n+\widetilde{\lambda}_{max}\bfI)^{-1}\right) \bfy \leq   \frac{\widetilde{\lambda}_{max}\|\bfy\|^2}{\sn^4+\sn^2\widetilde{\lambda}_{max}}.
\]
\cref{lem:KLbound} follows. 
\section{KL Divergence Gaussian Distributions}\label{app:KLdiv}
\subsection{KL divergence between multivariate Gaussian distributions}
We make use of the formula for KL divergences between multivariate Gaussian distributions in our proof of \cref{lem:avgKLbound}, and the univariate case in proposition \ref{prop:pointwise}. 

Recall the KL divergence from $p_1 \sim \mathcal{N}\left(\mathbf{m_1},\mathbf{S_1}\right)$ to $p_2 \sim \mathcal{N}\left(\mathbf{m_2},\mathbf{S_2}\right)$ both of dimension $N$ is given by
\begin{multline}\label{eqn:gausskldiverge}
      \KL{p_1}{p_2} = \frac{1}{2}\left(\Tr\left(\mathbf{S_2}^{-1}\mathbf{S_1}\right)+\log\left(\frac{\detbar{\mathbf{S_2}}}{\detbar{\mathbf{S_1}}}\right)\right. \\ \left. + \left(\mathbf{m_1} -\mathbf{m_2}\right)\transpose \mathbf{S_2}^{-1}\left(\mathbf{m_1} -\mathbf{m_2}\right)-N\right)
      \geq 0.
\end{multline}
The inequality is a special case of Jensen's inequality. 
\subsection{Proof of Upper Bound in \cref{lem:avgKLbound}}
In the main text we showed,
\begin{align*}
        &\Exp{y}{ \KL{Q}{\hat{P}}}  = \frac{t}{2\sn^2}+  \int \NormDist{\bfy; 0, {\bf K}_n} \\ &\qquad\times \log \left( \frac{\NormDist{\bfy; 0,{\bf K}_n}}{\NormDist{\bfy; 0,{\bf Q}_n}}\right)\calcd\bfy
\end{align*}
In order to complete the proof, we need to show that the second term on the right hand side is bounded above by $t/(2\sn^2)$. Using \Cref{eqn:gausskldiverge}:
\begin{multline}\label{eqn:avgkl}
\Exp{y}{ KL \left( Q \|\hat{P} \right)}  =  \frac{t}{2\sn^2} - \frac{N}{2}  +\frac{1}{2} \log\left(\frac{\left|{\bf Q}_n\right|}{\left|{\bf K}_n\right|}\right)\\ + \frac{1}{2}\Tr \left( {\bf Q}_n^{-1}({\bf K}_n)\right) \\
\leq  \frac{t}{2\sn^2} - \frac{N}{2}  +\frac{1}{2}\Tr \left({\bf Q}_n^{-1}({\bf Q}_n+\tilk)\right).
\end{multline}
The inequality follows from noting the log determinant term is negative, as ${\bf K}_n \succ {\bf Q}_n$ (i.e. ${\bf K}_n -{\bf Q}_n$ is positive definite). Simplifying the last term,
\begin{align*}
  \frac{1}{2}\Tr(\bfI) &+\frac{1}{2}\Tr\left({\bf Q}_n^{-1}\tilk)\right) \leq N/2 +t\lambda_1\left({\bf Q}_n^{-1}\right)/2 \\
    &\leq N/2+ t/(2\sn^2).
\end{align*}
The first inequality uses that for positive semi-definite symmetric matrices $\Tr(AB) \leq \Tr(A)\lambda_1(B)$ which is a special case of H\"older's inequality for Schatten norms. The final line uses that the largest eigenvalue of ${\bf Q}_n^{-1}$ is bounded above by $\sn^{-2}.$ Using this in \Cref{eqn:avgkl} finishes the proof.
\subsection{Proof of Proposition \ref{prop:pointwise}}
Defining $\epsilon = 2KL(q\|p),$
\begin{align}\label{eqn:1dkl}
\epsilon &= \frac{\sigma_1^2+(\mu_1-\mu_2)^2}{\sigma_2^2} - \log\left(\frac{\sigma^2_1}{\sigma^2_2}\right)-1  \\
&\geq x-\log(x)-1 \nonumber
\end{align}
where we have defined $x=\frac{\sigma^2_1}{\sigma^2_2}.$

Applying the lower bound $x - \log(x) -1 \geq (x-1)^2/2 - (x-1)^3/3,$ 
\[
\epsilon \geq (x-1)^2/2 - (x-1)^3/3.
\]

A bound on $\lvert x-1 \rvert$ that holds for all $\epsilon$ can then be found with the cubic formula. Under the assumption that $\epsilon < \frac{1}{5},$ we have $x-\log(x)<1.2$ which implies $x \in [0.493,1.77].$ For $x$ in this range, we have
\[
x - \log(x) -1 \geq (x-1)^2/3
\]
So,
\[
\lvert x-1 \rvert \leq \sqrt{3\epsilon}
\]
This proves that,
\[
1-\sqrt{3\epsilon}<\frac{\sigma_1^2}{\sigma_2^2}<1+\sqrt{3\epsilon}.
\]
From \Cref{eqn:1dkl} and $x-\log x>1$,
\[
\lvert \mu_1-\mu_2 \rvert \leq \sigma_2\sqrt{3\epsilon}.
\]
Using our bound on the ratio of the variances completes the proof of proposition \ref{prop:pointwise}.

 \section{Covariances for Interdomain Features}\label{app:covcalcs}
 We compute the covariances for eigenvector and eigenfunction inducing features. 
 \subsection{Eigenvector inducing features}
 Recall we have defined eigenvector inducing features by,
 \[
 u_m = \sum_{i=1}^N w_i^{(m)}f(\bfx_i).
 \]
 Then,
 \begin{align*}
     \cov(u_m,u_k) &= \Exp{}{\sum_{i=1}^N w_i^{(m)}f(\bfx_i)\sum_{j=1}^N w_j^{(k)}f(\bfx_j)} \\
     & = \sum_{i=1}^N w_i^{(m)}\sum_{j=1}^N w_j^{(k)}\Exp{}{f(\bfx_i)f(\bfx_j)}\\
     &= \sum_{i=1}^N w_i^{(m)}\sum_{j=1}^N w_j^{(k)}k(\bfx_i,\bfx_j).
 \end{align*}
 We now recognize this expression as $\mathbf{w}^{(m)\mathrm{\textsf{\tiny T}}}\Kff\mathbf{w}^{(k)}.$  Using the defining property of eigenvectors as well as orthonormality,
 \[
 \cov(u_m,u_k) = \lambda_k(\Kff)\delta_{m,k}.
 \]
 Similarly,
  \begin{align*}
     \cov(u_m,f(\bfx_i)) &= \Exp{}{\sum_{j=1}^N w_j^{(m)}f(\bfx_j)f(\bfx_i)} \\
     & = \sum_{j=1}^N w_j^{(m)}\Exp{}{f(\bfx_j)f(\bfx_i)}\\
     &= \sum_{j=1}^N w_j^{(m)}k(\bfx_j,\bfx_i).
 \end{align*}
 This is the $i^{th}$ entry of the matrix vector product $\Kff\mathbf{w}^{(m)}= \lambda_m(\Kff)\mathbf{w}^{(m)}_i.$
 \subsection{Eigenfunction inducing features}
 Recall we have defined eigenfunction inducing features by,
 \[
  u_m= \int \phi_m(\bfx)f(\bfx)p(\bfx)d\bfx.
 \]
 Then,
  \begin{align*}
     &\cov(u_m,u_k) \\&= \Exp{}{ \int \phi_m(\bfx)f(\bfx)p(\bfx)d\bfx \int \phi_k(\bfx')f(\bfx')p(\bfx')d\bfx'} \\
     & =  \int \phi_m(\bfx)p(\bfx)\int \phi_k(\bfx')\Exp{}{f(\bfx)f(\bfx')}p(\bfx')d\bfx'd\bfx \\
    & =  \int \phi_m(\bfx)p(\bfx)\int \phi_k(\bfx')k(\bfx,\bfx')p(\bfx') d\bfx'd\bfx.
 \end{align*}
 The expectation and integration may be interchanged by Fubini's theorem, as both integrals converge absolutely since $p(\bfx)$ is a probability density, the $\phi_m(\bfx)$ are in $L^2(\mcX)_p \subset L^1(\mcX)_p$ and $k$ is bounded. 
 
 We may then apply the eigenfunction property to the inner integral and orthonormality of eigenfunctions to the result yielding,
 \[
  \cov(u_m,u_k) = \lambda_k\int \phi_k(\bfx)\phi_m(\bfx)p(\bfx)d\bfx = \lambda_k \delta_{m,k}.
 \]
With similar considerations,
  \begin{align*}
     \cov(u_m,f(\bfx_i)) &= \Exp{}{ \int \phi_m(\bfx)f(\bfx)f(\bfx_i)p(\bfx)d\bfx} \\
     & =  \int \phi_m(\bfx)\Exp{}{f(\bfx)f(\bfx_i)}p(\bfx)d\bfx\\
    & =  \lambda_m\phi_m(\bfx_i).
 \end{align*}
\section{Discrete k-DPPs}\label{app:algorithm}
\subsection{Proof of Corollary \ref{cor:approxkdpp} from \Cref{lem:detsample}}
\begin{align}
    \Exp{Z\sim \nu}{t} &= \Exp{Z\sim \mu}{t} + (\Exp{Z\sim \nu}t - \Exp{Z\sim \mu}t) \nonumber \\
    & \hspace{-.5cm} \leq (M+1) \sum_{M+1}^N \lambda\left(\Kff\right) + \sum_{Z \in \binom{N}{M}} (\mu(Z)-\nu(Z))t(Z) \nonumber\\ 
    & \hspace{-.5cm}\leq (M+1) \sum_{M+1}^N \lambda\left(\Kff\right) + Nv \sum_{Z \in \binom{N}{M}} |\mu(Z)-\nu(Z)| \nonumber \\
    & \hspace{-.5cm}\leq (M+1) \sum_{M+1}^N \lambda\left(\Kff\right) + 2Nv\epsilon. \nonumber
\end{align}
The first inequality follows from \cref{lem:detsample}. The second uses the triangle inequality replace $t(Z)$ with a bound on its maximum. The final line uses one of the definitions of total variation distance for discrete random variables. 
\subsection{Sampling Approximate k-DPPs}
\citet{belabbas_spectral_2009} proposed using the Metropolis method for approximate sampling from a k-DPP. Several recent works have shown that a natural Metropolis algorithm on k-DPPs mixes quickly. In particular, \citet{pmlr-v49-anari16} considers the following algorithm:
 \begin{algorithm}[tb]
  \caption{MCMC algorithm for sampling $\epsilon$ k-DPP (A)}
  \label{alg:det_init}
  \begin{algorithmic}
    \STATE {\bfseries Input:} Training inputs $\bfX=\{\bfx_i\}_{i=1}^N$, number of points to choose, $M$, kernel $k$.
  \STATE {\bfseries Returns:} A sample of $M$ inducing points drawn proportional to the determinant of $\bfK_{Z,Z}$

  \STATE Initialize $M$ columns greedily in an iterative fashion call this set $S_0.$
  \FOR{$r<R$}
  \STATE Sample $i$ uniformly from $S_{r}$ and $j$ uniformly from $\bfX\setminus S_{r}.$ Define $T= S\setminus \{i\}\cup\{j\},$
  \STATE Compute $p_{i \to j}:= \frac{1}{2}\min\{1, \text{det}(\bfK_T)/\text{det}(\bfK_{S_r})\}$
  \STATE With probability $p_{i \to j}$ $S_{r+1}=T$ otherwise, $S_{r+1}=S$
  \ENDFOR
  \STATE {\bfseries Return:} $S_{R}
  $
\end{algorithmic}
\end{algorithm}
\begin{theorem}[\citet{pmlr-v49-anari16}, Theorem 2]
Let A denote \cref{alg:det_init}. Let $\nu^R$ denote the distribution induced by $R$ steps of A. Let $R(\epsilon)$ denote the minimum $R$ such that $\|\mu-\nu^R\|_{TV}<\epsilon,$ where $\mu$ is a k-DPP on some kernel matrix $\Kff.$ Then 
\[
R(\epsilon) \leq NM\log\left(\frac{\binom{N}{M}M!}{\epsilon}\right) \leq NM^2\log{N} + NM\log\frac{1}{\epsilon}.
\]
\end{theorem}
Taking $\epsilon$ to be any fixed inverse power of $N,$ (i.e. $\epsilon = N^{-\gamma},$ will make the second term $\BigO(NM\log(N)),$ while by taking $\gamma$ large (e.g. greater than 2), we can make $2Nv\epsilon$ small. 

The total cost of the algorithm is determined by the cost of the greedy initialization, plus $R(\epsilon)$ times the per iteration cost of the algorithm. A naive implementation of the greedy initialization requires $O(NM^4)$ time and $\BigO(NM)$ memory, simply by computing the determinant of each of the $N-m$ possible ways to extend the current subset (faster implementations are possible, but this suffices for our purposes). We assume that this is implemented in such a way that at the end of the initialization we have access to $\text{det}(\bfK_{S_0})$ and a Cholesky factorization $S_0=\mathbf{L_0}\mathbf{L_0}\transpose.$

We take as an inductive hypothesis that at iteration $r$ of the algorithm, we know $\text{det}(\bfK_{S_{r}})$ and a Cholesky factorization, $\bfK_{S_r}=\mathbf{L}_{r}\mathbf{L}\transpose_{r}.$ We then need to show we can compute a Cholesky factorization and determinant of $\bfK_T$ in $\BigO(M^2).$ Given the Cholesky factorization of $T,$ $\text{det}(\bfK_T)$ can be computed $\BigO(M)$ as it is the product of the square of the diagonal elements. It therefore remains to consider the calculation of $\mathbf{L}_{T},$ a Cholesky factor of $\bfK_T.$

The computation of $L_T$ proceeds in two steps: first we compute $\mathbf{L}_{S\setminus i}$ using $\mathbf{L}_{S}.$ 
\begin{align*}
    \mathbf{L}_{S}\!=\!\begin{bmatrix} \mathbf{L}_{1,1} & 0 & 0 \\
     \mathbf{\ell}_{2,1} & \ell_{2,2} & 0 \\
     \mathbf{L}_{3,1} & \mathbf{\ell}_{3,2} & \mathbf{L}_{3,3}\end{bmatrix} \text{,}   \mathbf{K}_{S}\!=\!\begin{bmatrix} \mathbf{K}_{1,1} & \mathbf{k}_{2,1} &\mathbf{K}_{1,3} \\
     \mathbf{k}_{2,1} & k(\bfx_i,\bfx_i) & \mathbf{k}_{2,3}\\
     \mathbf{K}_{3,1} & \mathbf{k}_{3,2} & \bfK_{3,3}\end{bmatrix}
\end{align*}
A direct calculation shows,
\[
L_{S\setminus i}=\begin{bmatrix} \mathbf{L}_{1,1} & 0  \\
     \mathbf{L}_{3,1} & \mathbf{L'}_{3,3}\end{bmatrix} 
\]
where $\mathbf{L'}_{3,3}\mathbf{L'}_{3,3}\transpose= \mathbf{L}_{3,3}\mathbf{L}_{3,3}\transpose+\mathbf{\ell}_{3,2}\mathbf{\ell}_{3,2}\transpose.$ This is a rank one-update to a Cholesky factorization, and can be performed using standard methods in $\BigO(M^2).$ 

We now need to extend a Cholesky factorization from $S\setminus i$ to $T,$ which involves adding a row. 
\begin{align*}
    \mathbf{L}_{T} = \begin{bmatrix} \mathbf{L}_{S\setminus i} & 0 \\
     \mathbf{c} & d \end{bmatrix}
\end{align*}
    with $\mathbf{c} = \mathbf{L}_{S\setminus i}^{-1} \mathbf{k_{S\setminus i, j}}$ and $d= \sqrt{k(\bfx_j,\bfx_j)-\mathbf{k_{S\setminus i, j}\transpose\mathbf{L}_{S\setminus i}^{-1\mathrm{\textsf{\tiny T}}}  \mathbf{L}_{S\setminus i}^{-1} \mathbf{k_{S\setminus i, j}}}}.$ All of these calculations can be computed in $\BigO(M^2)$ completing the proof of the per iteration cost.

\section{Proof of Corollaries}\label{app:cors}
\subsection{Corollary \ref{cor:gaussklM}}
From \cref{thm:generalexplicit}, with probability $1-\delta$ and this choice of $\epsilon$
\begin{align}\label{eqn:restate}
	\KL{Q}{\hat{P}} &\leq \frac{C(M+1)}{2\sn^2\delta}\left(1+ \frac{\|\bfy\|^2_2}{\sn^2}\right) \nonumber\\&+ N^{-\gamma}(R/\sn^2+1/N)
\end{align}
Take $M=\frac{(3+\gamma)\log(N)+\log D}{\log(B^{-1})}.$ If $M\geq N$ the KL-divergence is zero and we are done. Otherwise,  $C(M+1)<N^2\sum_{i=M+1}^\infty \lambda_i.$ By the geometric series formula, 
\[
\sum_{i=M+1}^\infty \lambda_i = v\frac{\sqrt{2a}}{\sqrt{A}}\frac{B^{M}}{1-B}
\]
Now, $B^M = N^{-3-\gamma}D^{-1},$ so
\[
\sum_{i=M+1}^\infty \lambda_i = \delta N^{-3-\gamma},
\]
implying $\frac{C(M+1)}{2\delta\sn^2}<N^{-1-\gamma}.$ Using this in \cref{eqn:restate} completes the proof.

\subsection{Corollary \ref{cor:gaussMD}}
It is sufficient to consider the case of isotropic kernels and input distributions.\footnote{For the general case, the eigenvalues can be bounded above by  constant times the eigenvalues of an operator with an isotropic kernel with all lengthscales equal to the shortest kernel lengthscale and the input density standard deviation set to the largest standard deviation of any one-dimensional marginal of $p(\bfx).$} From \cite{seeger2008information} in the isotropic case (i.e. $B_i=B_j=:B$ for all $i,j\leq D),$
\[
\lambda_{s+D-1} \leq \left(\frac{2a}{A}\right)^{D/2} B^{s^{1/D}}.
\]
Define $\tilde{M}= M+D-1,$ to be the number of features used for inference.

\begin{align}
\sum_{s=\tilde{M}+1}^\infty \limits\lambda_s &\leq  \left(\frac{2a}{A}\right)^{D/2} \sum_{s=M+1}^\infty B^{s^{1/D}} \\
& < \left(\frac{2a}{A}\right)^{D/2} \int_{s=M}^\infty B^{s^{1/D}} ds =: \mathcal{I}.
\end{align}
In the second line, we use that $B<1,$ so $B^{s^{1/D}}$ obtains its minimum on the interval $s\in[s',s'+1]$ at the right endpoint (i.e. monotonicity). We now define $\alpha = -\log(B).$ So,
\begin{align}
\mathcal{I} &= \left(\frac{2a}{A}\right)^{D/2} \int_{s=M}^\infty \exp\left(-\alpha s^{1/D}\right) ds \\
&= \left(\frac{2a}{A}\right)^{D/2} \alpha^{-D}D\int_{t=\alpha M^{1/D}}^\infty e^{-t}t^{D-1} dt
\end{align}
In the second line we made the substitution $t=\alpha s^{1/D},$ so $ds = \alpha^{-D}Dt^{D-1}.$
We now recognize,
\[
\int_{t=\alpha M^{1/D}}^\infty e^{-t}t^{D-1} dt 
\]
as an incomplete gamma function, $\Gamma(D,\alpha M^{1/D}).$ From \citet[8.352]{gradshteyn2014table}, 
\[
\Gamma(D,y) = (D-1)!e^{-y}\sum_{k=0}^{D-1}\frac{y^k}{k!}
\]
So,
\begin{align}
  \mathcal{I} &= \left(\frac{2a}{A}\right)^{D/2} \alpha^{-D}D!e^{-\alpha M^{1/D}}\sum_{k=0}^{D-1}\frac{\alpha^k M^{k/D}}{k!}
\end{align}
As $M$ grows as a function of $N$ and $D$ is fixed, for $N$ large $D\leq \alpha M^{1/D}.$ This implies that that the largest term in the sum on the right hand side is the final term, so 
\begin{align}
  \mathcal{I} \leq  \left(\frac{2a}{A}\right)^{D/2}\alpha^{-1}D^2e^{-\alpha M^{1/D}}M^{(D-1)/D}.
  \end{align}
Choose $M = \frac{1}{\alpha}\log\left(N^{\gamma'}\left(\frac{2a}{A}\right)^{D/2}D^2\alpha^{-1}\right)^D.$ Then
\[
\left(\frac{2a}{A}\right)^{D/2}\alpha^{-1}D^2e^{-\alpha M^{1/D}} = N^{-\gamma'}
\]
and 
\[
M^{(D-1)/D} < M = \frac{1}{\alpha}\log\left(N^{\gamma'}\left(\frac{2a}{A}\right)^{D/2}D^2\alpha^{-1}\right)^D,
\]
so
\[
  \mathcal{I} \leq  \frac{1}{\alpha}\log\left(N^{\gamma'}\left(\frac{2a}{A}\right)^{D/2}D^2\alpha^{-1}\right)^D N^{-\gamma'}.
\]
For any fixed $D,$ for this choice of $M$ for any $\varepsilon>0$ for $N$ large,
\[
\mathcal{I} = \BigO\left(N^{-\gamma'+\varepsilon}\right).
\]
By choosing $\gamma' > 3+\epsilon',$ for some fixed $\epsilon'>0$ the proof is complete, using a similar argument as the one used in the proof of the previous corollary.

Note that using the bound proven in Theorem 2 (the tightest of our bounds) the exponential scaling in dimension is unavoidable. If both $k$ and $p(\bfx)$ are isotropic, then the eigenvalue $(\frac{2a}{A})^DB^m$ appears $\binom{m+D-1}{D-1}$ times. This follows from noting that this is the number of ways to write $m$ as a sum of $D$ non-negative integers. Using an identity and standard lower bound for binomial coefficients $\sum_{i=1}^K \binom{m+D-1}{D-1}= \binom{K+D}{D}\geq \left(\frac{K+D}{D}\right)^D > C(D)K^D.$ for some constant depending on D, $C(D).$ Using Theorem 2 we need to choose $M$ such that $\lambda_M = \BigO(1/N).$ This means choosing $K \gg \log(N)$ in the sum above, leading to at least $\alpha \log^D(N)$ features being needed for some constant $\alpha$. The constant in this lower bound decays rapidly with $D,$ while the constant in the upper bound does not. Better understanding this gap is important for understanding the performance of sparse Gaussian process approximations in high dimensions.

If the data actually lies on a lower dimensional manifold, we conjecture the scaling depends mainly on the dimensionality of the manifold. In particular, if the manifold is linear and axis-aligned, then the kernel matrix only depends on distances along the manifold (not in the space it is embedded in) so the eigenvalues will not be effected by the higher dimensional embedding. We conjecture that similar properties are exhibited when the data manifold is nonlinear.

\section{Smoothness and Sacks-Ylivasker Conditions}\label{app:SY}
In many instances the precise eigenvalues of the covariance operator are not available, but the asymptotic properties are well understood. A notable example is when the data is distributed uniformly on the unit interval. If the kernel satisfies the Sacks-Ylivasker conditions of order $r$:
\begin{itemize}
    \item $k(x, x')$ is $r$-times continuously differentiable on $[0,1]^2$ Moreover, $k(x,x')$ has continuous partial derivatives up to order $r+2$ times on $(0,1)^2 \cap (x>x')$ and $(0,1)^2 \cap (x< x').$ These partial derivatives can be continuously extended to the closure of both regions.
    \item Let $L$ denote $k^{(r,r)}(x, x')$, $L_{+}$ denote the restriction of $L$ to the upper triangle and $L_{-}$ the restriction to the lower triangle, then $L_{+}^{(1,0)}<L_{-}^{(1,0)}$ on the diagonal $x=x'.$
    \item  $L^{(2,0)}_{+}(s, \cdot)$ is an element of the RKHS associated to $L$ and has norm bounded independent of $s.$
\end{itemize}
Notably, Mat\'ern half integer kernels of order $r+1/2$ meet the S-Y condition of order $r.$ See \citet{rittermulti1995} for a more detailed explanation of these conditions and extensions to the multivariate case. 

\end{document}